\documentclass{article}

\usepackage{microtype}
\usepackage{graphicx}
\usepackage{subfigure}
\usepackage{booktabs} 
\usepackage{amsfonts} 
\usepackage{amssymb}
\usepackage{amsmath}
\usepackage{sansmath}
\usepackage{MnSymbol,wasysym}
\usepackage{subfigure}
\usepackage{amsthm}
\usepackage{tcolorbox}
\newtheorem{theorem}{Theorem}
\newtheorem{definition}[theorem]{Definition}

\newtheorem{proposition}[theorem]{Proposition}
\newtheorem{lemma}[theorem]{Lemma}

\usepackage{mathtools}
\usepackage{enumitem}
\newcommand{\eat}[1]{}

\usepackage{hyperref}

\usepackage{bm}
\usepackage{multirow}
\newtcolorbox[list inside=tcb]{mytcb}[1][]{#1}

\usepackage{caption}

\usepackage[accepted]{icml2022}

\icmltitlerunning{Rethinking Graph Neural Networks for Anomaly Detection}

\begin{document}
\setcounter{page}{1}
\twocolumn[
\icmltitle{Rethinking Graph Neural Networks for Anomaly Detection}




\begin{icmlauthorlist}
\icmlauthor{Jianheng Tang}{to1,to2}
\icmlauthor{Jiajin Li}{goo}
\icmlauthor{Ziqi Gao}{to1,to2}
\icmlauthor{Jia Li}{to1,to2}
\end{icmlauthorlist}

\icmlaffiliation{to1}{Hong Kong University of Science and Technology (Guangzhou)}
\icmlaffiliation{to2}{Hong Kong University of Science and Technology}
\icmlaffiliation{goo}{Stanford University}

\icmlcorrespondingauthor{Jia Li}{jialee@ust.hk}
\icmlkeywords{Machine Learning, ICML}

\vskip 0.3in
]

\printAffiliationsAndNotice{} 

\begin{abstract}

Graph Neural Networks (GNNs) are widely applied for graph anomaly detection. As one of the key components for GNN design is to select a tailored spectral filter, we take the first step towards analyzing anomalies via the lens of the graph spectrum. Our crucial observation is the existence of anomalies will lead to the `right-shift' phenomenon, that is, the spectral energy distribution concentrates less on low frequencies and more on high frequencies.  This fact motivates us to propose the Beta Wavelet Graph Neural Network (BWGNN). Indeed, BWGNN has spectral and spatial localized band-pass filters to better handle the `right-shift' phenomenon in anomalies. We demonstrate the effectiveness of BWGNN on four large-scale anomaly detection datasets. Our code and data are released at \url{https://github.com/squareRoot3/Rethinking-Anomaly-Detection}.

\end{abstract}

\section{Introduction}\label{sec1}

An anomaly or an outlier is a data object that deviates significantly from the majority of the objects, as if it was generated by a different mechanism~\cite{textbook}. As a well-established problem, anomaly detection has received much attention due to its vast applicable tasks, e.g., cyber security \cite{anomaly_cyber}, fraud detection \cite{anomaly_financial}, health monitoring \cite{anomaly_health}, device failure detection \cite{sipple2020interpretable}, to name a few.
As graph data becomes ubiquitous in the Web era, graph information often plays a vital role in identifying fraudulent users or activities, e.g., friendship relations in a social network and transaction records on a financial platform. Consequently,  as a crucial research direction, graph-based anomaly detection is necessary to be further explored \cite{noble2003graph, anomaly_survey}.

Nowadays, Graph Neural Networks (GNNs) act as popular approaches for mining structural data, and are naturally applied for the graph anomaly detection task ~\cite{ONE, kumar2018rev2}. 
Unfortunately, the vanilla GNNs are not well-suited for anomaly detection and suffer from the over-smoothing issue~\cite{li2018in, wu2019simplifying}. 
When GNN aggregates information from the node neighborhoods, it also averages the representations of anomalies and makes them less distinguishable. Thus, by intentionally connecting with large amounts of benign neighborhoods, the anomalous nodes may attenuate their suspiciousness, which results in the poor performance of the plain GNNs.

To remedy the issue, several GNN models have been proposed. We categorize the existing methods into three classes. That is, (1) applying attention mechanisms to correlate different neighbors through various views \cite{wang2019semi, cui2020deterrent, liu2021intention}, (2) using the resampling strategy to aggregate neighborhood information selectively \cite{dou2020enhancing, liu2020alleviating, liu2021pick}, (3) designing auxiliary losses to enhance the network training power \cite{dominant, zhao2020error, zhao2021synergistic}.
All these methods analyze the anomaly detection from the graph spatial domain. There are few works that address this problem from the spectral domain. Nevertheless, choosing a tailored spectral filter is a key component of GNN design, as the spectral filter determines the expressive power of GNN \cite{balcilar2020analyzing, he2021bernnet}.

\begin{figure*}[t]
\centering
  \includegraphics[width=1.0\linewidth]{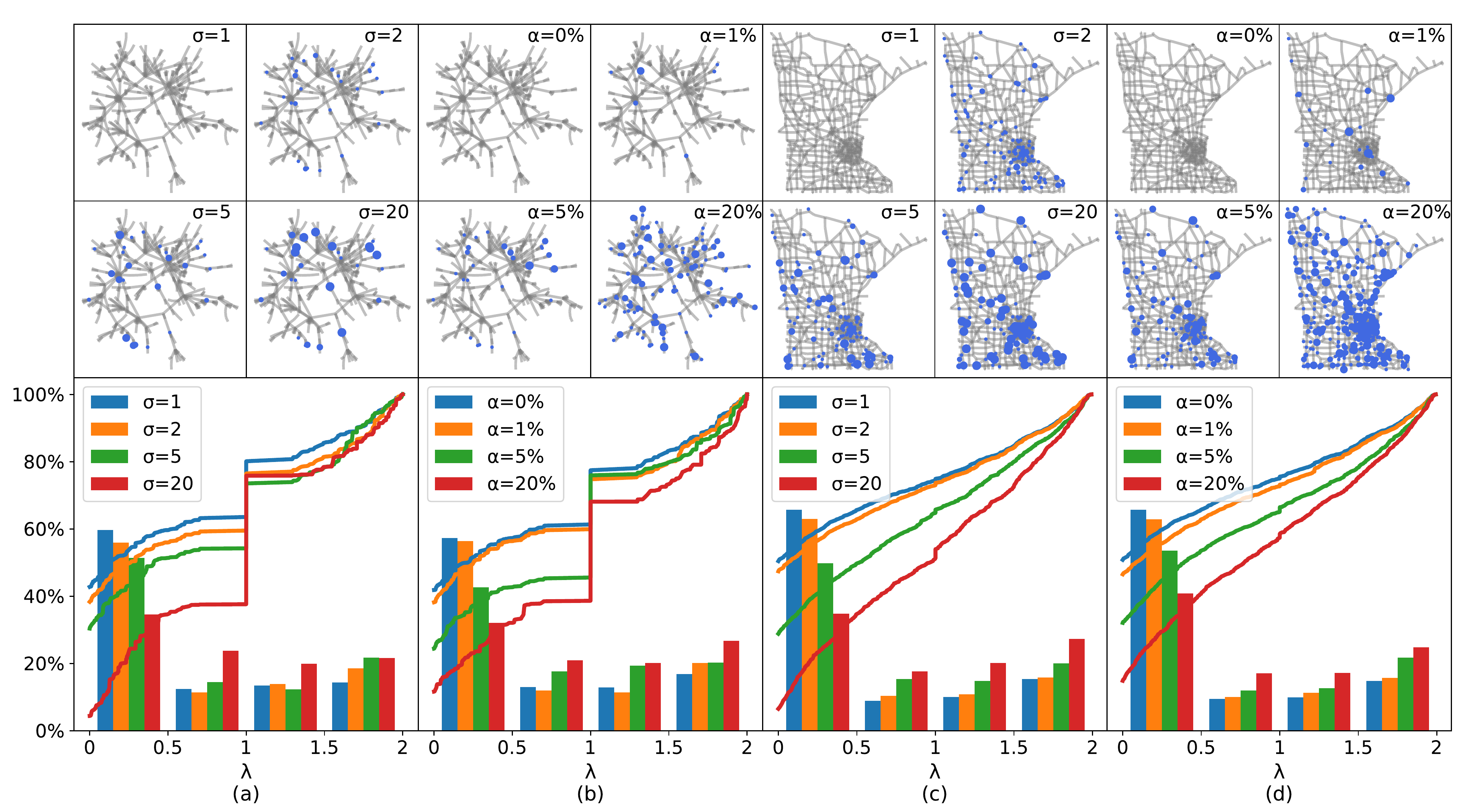}
\caption{The effect of graph anomalies in the spatial domain (top) and spectral domain (bottom) with different anomaly degrees. The cases (a,c) are related to different standard deviation of anomalies (type (i), $\sigma=1,2,5,20$) while the cases (b,d) are about different fraction of anomalies (type (ii),  $\alpha=0\%, 1\%, 5\%, 20\%$).}
\label{fig:intro}
\end{figure*}

To fill the above gap, we want to answer the question in this paper--- \emph{how to choose a tailored spectral filter in GNN for anomaly detection?} We take the first step towards analyzing anomalies via the lens of the graph spectrum (i.e., after the graph Fourier transform of node attributes). From Figure~\ref{fig:intro}, we observe that low-frequency energy is gradually transferred to the high-frequency part, when the degree of anomaly becomes larger. We summarize this phenomenon as the `right-shift' of spectral energy distribution, and further prove it on a Gaussian anomaly model in a rigorous manner. We validate the `right-shift' phenomenon in a variety of graphs with synthetic or real-world anomalies. Through our analysis, we justify the necessity of spectral localized band-pass filters in graph anomaly detection.

Based on these findings, we propose Beta Wavelet Graph Neural Network (BWGNN) to better tackle the `right-shift' phenomenon of graph anomalies. Although the frequency responses of existing works~\cite{wu2021beyond,he2021bernnet,bo2021beyond} can cover nearly all frequency profiles, these GNNs with adaptive filters cannot guarantee to be band-pass and spectral-localized \cite{balcilar2020analyzing}. In fact, the design of frequency response in those works are not tailored to anomaly detection. Therefore, we invoke Hammond's graph wavelet theory \cite{hammond2011wavelets} to develop our new graph neural network architecture. In contrast with the widely used Heat kernels \cite{wavelet_jure, GWNN,gdnli}, the crux of our method is to propose the Beta kernel to  address higher frequency anomalies via multiple flexible, spatial/spectral-localized, and band-pass filters.

To further facilitate this line of research, we release two large-scale real-world datasets as new graph anomaly detection benchmarks, including the T-Finance dataset based on a transaction network and the T-Social dataset based on a social network. Together with them, we conduct extensive experiments on four datasets in both supervised and semi-supervised settings. The proposed BWGNN shows superior performance over widely used graph neural networks and state-of-the-art graph anomaly detection methods.

\section{Spectral Analysis of Graph Anomaly}\label{sec:analysis}

To start with, we provide some necessary preliminaries of graph anomaly detection in Section \ref{sec:analysis1}. In Section \ref{sec:analysis2}, 
we give the theoretical insights of the `right-shift' phenomenon and rigorously prove it on a Gaussian anomaly model. Towards that end, we validate our findings on graphs with synthetic anomalies in Section \ref{sec:analysis3}, and on real-world anomaly detection datasets in Section \ref{sec:analysis4}.

\subsection{Preliminaries}\label{sec:analysis1}
\noindent\textbf{Attributed Graph} We define an attributed graph as $\mathcal G=\{\mathcal V,\mathcal E, \bm X\}$, where $\mathcal V=\{v_1, v_2, \cdots, v_N\}$ is the set of $N$ nodes, $\mathcal E=\{e_{ij}\}$ is the set of edges, and $e_{ij}=(v_i, v_j)$ represents an unweighted edge between nodes $v_i$ and $v_j$. Let $\bm A$ be the corresponding adjacency matrix, $\bm D$ be the degree matrix with $\bm D_{ii}=\sum_j\bm A_{ij}$. Each node $v_i$ has a $d$-dimensional feature vector $\bm X_i \in \mathbb R^d$, and the set of all node features is $\bm X = \{ \bm X_1, \bm X_2, \cdots, \bm X_N \}$. In some applications, $\mathcal G$ is a multi-relational graph and multiple edge sets represent different relations between nodes.

\noindent\textbf{Graph-based Anomaly Detection} \eat{Let $\mathcal G$ be an attributed graph defined above and}Let $\mathcal V_a, \mathcal V_n$ be two disjoint subsets of $\mathcal V (i.e., \mathcal V_a \cap \mathcal V_n = \varnothing)$, where $\mathcal V_a$ represents all the nodes labeled as \textbf{anomalous} and $\mathcal V_n$ represents all \textbf{normal} nodes. Graph-based anomaly detection is to classify unlabeled nodes in $\mathcal G$ into the normal or anomalous categories given the information of the graph structure $\mathcal E$, node features $\bm X$, and partial node labels $\{\mathcal V_a, \mathcal V_n\}$. In this paper, we focus on node anomalies and assume that all edges in $\mathcal G$ are trusted, leaving structural anomalies for future work. Usually, there are far more normal nodes than anomalous nodes $(|\mathcal V_a| << |\mathcal V_n|)$, thus graph-based anomaly detection can be regarded as an imbalanced binary node classification problem. The main difference is that anomaly detection focuses more on the unusual and deviated patterns in the dataset.

\subsection{Theoretical Results}\label{sec:analysis2}

\paragraph{Problem Setup} Let the Laplacian matrix $\bm L$ be defined as $\bm D - \bm A$ (regular) or as $\bm I-\bm D^{-1/2}\bm A\bm D^{-1/2}$ (normalized), where $\bm I$ is an identity matrix. $\bm L$ is a symmetric matrix with eigenvalues, i.e., $0= \lambda_1 \leq \cdots \leq \lambda_N $ and a corresponding orthonormal basis of eigenvectors $\bm U=(\bm u_1, \bm u_2, \cdots, \bm u_N)$. Except for two endpoints $\lambda_1$ and $\lambda_N$, we can split other eigenvalues into the low frequencies $\{\lambda_1,\lambda_2,\cdots,\lambda_k\}$ and high frequencies $\{\lambda_{k+1},\lambda_{k+2},\cdots,\lambda_N \}$ with an arbitrary threshold $\lambda_k$.

Assume that $\bm x = (x_1, x_2,\cdots,x_N)^T \in \mathbb R^N$ is a signal on $\mathcal G$\eat{ (a scalar feature for every node)}, and $\bm{\hat x}=(\hat x_1, \hat x_2,\cdots,\hat x_N)^T=\bm U^T \bm x$ is the graph Fourier transform of $\bm x$. We denote $\hat x_k^2/\Sigma_{i=1}^N \hat x_i^2$ as the \emph{spectral energy distribution} at $\lambda_k (1\leq k \leq N)$. We summarize our theoretical and empirical findings as

\fbox{\begin{minipage}{22.6em}
The existence of anomalies leads to the `right-shift' of spectral energy, which means the spectral energy distribution concentrates less in low frequencies and more in high frequencies.
\end{minipage}}

To further justify the insight,  we prove this interesting finding based on a probabilistic anomaly model \cite{textbook, grubbs1969procedures}. The graph features are assumed to be identically independent drawn from a Gaussian distribution, i.e., $\bm x \sim \mathcal{N}(\mu e_N, \sigma^2 \bm I_N)$, where $\mu e_N$ is an all-the-one vector. Here, the coefficient of variation $\sigma/|\mu|$ can be regarded as the degree of anomalies in $\bm x$, which is indeed a commonly used measure to describe the dispersion of a distribution~\cite{kendall1946advanced, bedeian2000use}. 
When the fraction of anomalies in $\bm x$ increases, $\sigma/|\mu|$ becomes larger and indicates a more considerate degree of anomaly. Alternatively, if the distance between anomalies and the mean vector becomes larger, $\sigma/|\mu|$ also increases, representing a larger degree of anomalies. 

To quantify how the spectral energy distribution changes with respect to the degree of anomalies in $\bm x$, we introduce a metric --- energy ratio as below: 
\begin{definition}[Energy Ratio]\label{def:1} For any $1 \leq k \leq N-1$, we define $k$-th low-frequency energy ratio as the accumulated energy distribution in the first $k$ eigenvalues:
\begin{equation*}
\eta_k(\bm x,\bm L) = \frac{\sum_{i=1}^k \hat x_i^2}{\sum_{i=1}^N \hat x_i^2}. 
\end{equation*}
\end{definition}
A larger $\eta_k$ indicates that a larger part of the energy boils down to the first $k$ eigenvalues. In the following proposition, we shed light on how the degree of anomalies in $\bm x$ will affect $\eta_k$.
\begin{proposition}\label{prop:1}
If $|\mu| \neq 0$ and $\bm L = \bm D - \bm A$, the expectation of the inverse of low-frequency energy ratio $\mathbb{E}_{\bm x}[1/\eta_k(\bm x,\bm L)]$ is monotonically increasing with the anomaly degree $\sigma/|\mu|$. 
\end{proposition}

Proposition \ref{prop:1} indicates that an increased degree of anomaly enforces the spectral energy distribution to concentrate less on the low-frequency eigenvalues. The proof details can be found in Appendix \ref{sec:append1}.

Unfortunately, the calculation of spectral energy ratio requires the eigen-decomposition of the graph Laplacian, which is time-consuming on large-scale graphs. To navigate such pitfalls, we introduce a more computational amenable metric to generally measure the effect of graph anomalies in the spectral domain.

\begin{definition}[High-frequency Area] \label{def:2}
Suppose that the low-frequency energy ratio curve $f(t)$ is defined as $f(t)=\eta_k(x,\bm L)$ where $t \in [\lambda_{k}, \lambda_{k+1})$ and $1 \leq k \leq N-1$. The area between $f(t)$ and $g(t)=1$ is defined as the high-frequency area: $S_{\textnormal{high}}=\int_0^{\lambda_N}1-f(t)dt$.
\end{definition}

Without the need of eigen-decomposition, $S_{\text{high}}$ can be computed by simple elementary manipulations:
\begin{equation}
\label{eq:high_freq}
S_{\text{high}}=\frac{\sum_{k=1}^N\lambda_k \hat x^2_k}{\sum_{k=1}^N \hat x_k^2}=\frac{\bm x^T\bm L \bm x}{\bm x^T\bm x}.
\end{equation}
We refer the reader to Appendix \ref{sec:append2} for further details. According to \eqref{eq:high_freq}, spectral energy on low frequencies contributes less to $S_{\text{high}}$ after multiplying a small eigenvalue. For example, $S_{\text{high}}=0$ if the entire spectral energy concentrates on $\lambda_1=0$, and $S_{\text{high}}$ increases when spectral energy shifts to larger eigenvalues. Therefore, we can use the change of $S_{\text{high}}$ to describe the `right-shift' phenomenon in the whole spectrum. 

Furthermore, \eqref{eq:high_freq} reveals that the energy distribution of $\bm x$ in the spectral domain is closely related to the smoothness of $\bm x$ in the spatial domain --- $\bm x^T\bm L \bm x$ will be smaller if signal $\bm x$ has closer values between connected nodes, which also reflects a smaller anomaly degree. In Appendix \ref{sec:append1}, we additionally proof that $S_{\text{high}}$ is monotonically increasing with the anomaly degree $\sigma/|\mu|$, which matches the result in Proposition \ref{prop:1}.

\begin{figure}[t]

\centering
  \includegraphics[width=1.0\columnwidth]{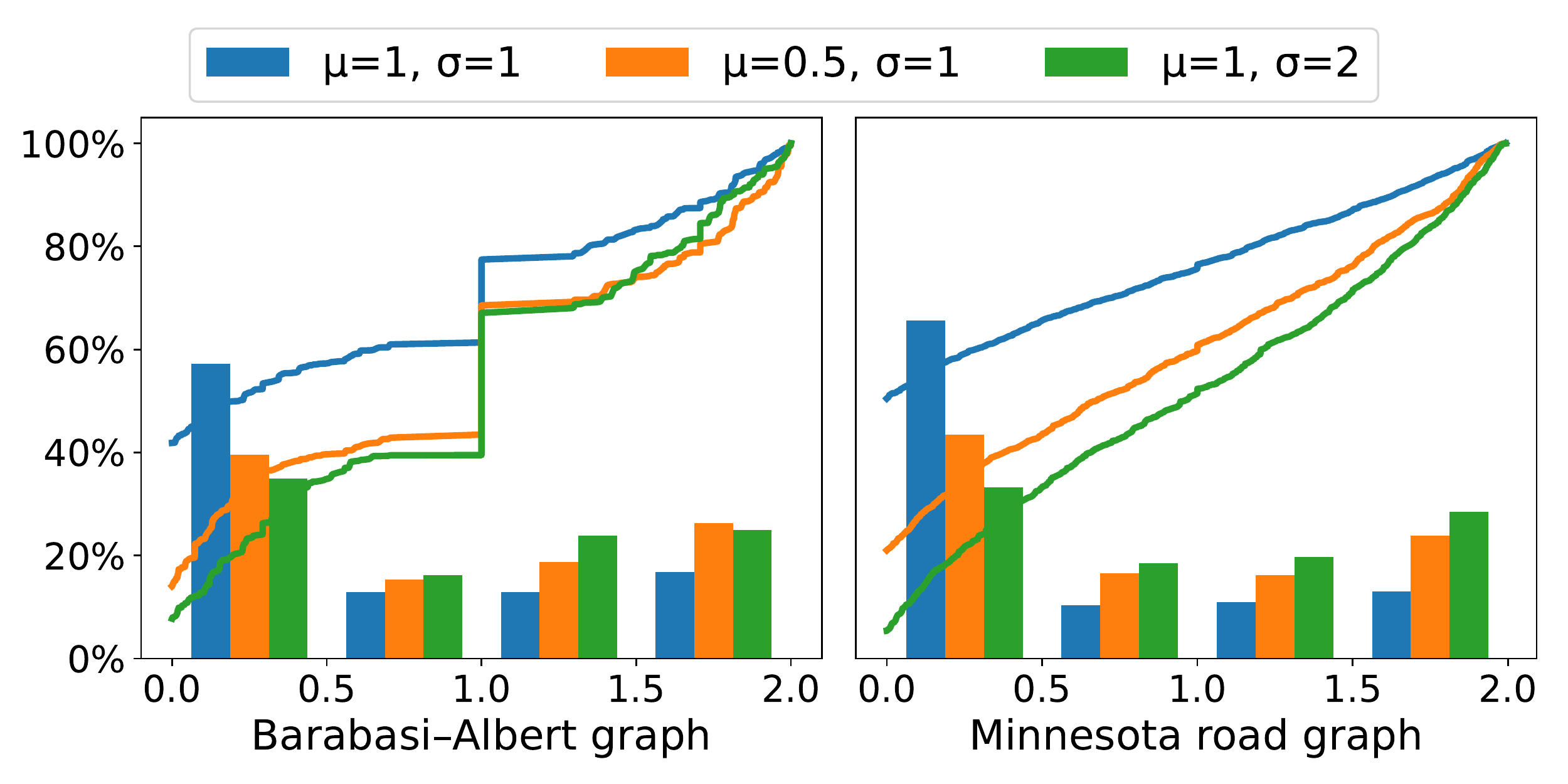}
\vspace{-4mm}
\caption{Spectral analysis of graph signal $\bm x \sim \mathcal{N}(\mu e_N, \sigma^2 \bm I_N)$.}
\label{fig:analysis}
\end{figure}

\begin{table*}[t!]
  \caption{Summary of dataset statistics and the comparison of anomaly effect on spectral energy distributions. The number in the parentheses is the lowest value of $\Delta S_{\text{high}}$ on each dataset.}
  \vskip 0.1in

\centering
\begin{tabular}{r|cccc|cc}
\toprule
              & \multicolumn{4}{c|}{Statistics}                  & \multicolumn{2}{c}{Anomaly Effect ($\Delta S_{\text{high}}$) }                                               \\
Dataset       & \# Nodes   & \# Edges    & Anomaly(\%) & \# Features & Drop-Anomaly & Drop-Random \\\midrule
Amazon        & 11,944    & 4,398,392  & 6.87\%      & 25       & \textbf{-0.59\% (-4.19\%) }& 0.15\% (-0.39\%)                                          \\
YelpChi       & 45,954    & 3,846,979  & 14.53\%     & 32       &  \textbf{-1.61\% (-15.4\%)} & -0.04\% (-1.14\%)                                      \\
T-Finance & 39,357    & 21,222,543 & 4.58\%      & 10       &  \textbf{-0.34\% (-0.92\%)} & 0.09\% (-0.01\%)                                          \\
T-Social      & 5,781,065 & 73,105,508 & 3.01\%      & 10       &  \textbf{-5.36\% (-13.2\%)} & 0.08\% (-0.02\%)                                          \\
\bottomrule
\end{tabular}

  \label{tab:data}
\end{table*}

\subsection{Validation on Synthetic Anomalies}\label{sec:analysis3}

To present our theoretical findings in a more intuitive way, we demonstrate it on different graphs with synthetic anomalies. In Figure~\ref{fig:intro}, we study the effect of anomalies on two types of graph: a Barabási–Albert graph with 500 nodes on Figure \ref{fig:intro} (a)-(b) and a Minnesota road graph with 2,642 nodes \cite{perraudin2014gspbox} in Figure \ref{fig:intro} (c)-(d).

We first assume normal and anomalous nodes follow different distributions to better visualize the change of spectral energy under different anomaly degrees. For both graphs, the one-dimensional feature of each normal node is drawn from $\mathcal{N}(1, 1)$, while the anomalous one is drawn from $\mathcal{N}(1, \sigma^2)$ (i.e., $\sigma >1$). We analyze two variations of anomalies: (i) the fraction of anomalies is fixed to 5\%, and the standard deviation of anomalies changes (i.e., $\sigma=1,2,5,20$); and (ii) the standard deviation of anomalies is fixed to 5 and the fraction of anomalies changes (i.e., $\alpha=0\%, 1\%, 5\%, 20\%$). In the top half of Figure \ref{fig:intro}, we use blue circles to represent anomalous nodes in the spatial domain. Bigger blue nodes indicate a larger degree of anomaly. In the bottom half of Figure \ref{fig:intro}, we display the energy distribution of $\bm x$ in the spectral domain with various anomaly degrees. The colored histograms reflect the proportion of spectral energy in different eigenvalue intervals such as $[0,0.5)$, and the curves represent $f(t)$ in Definition \ref{def:2}.

We take the Barabási–Albert graph as an example. When the fraction of anomalies is 0\%, most of the energy (60\%) locates in the low frequency region $(\lambda < 0.5)$.  By contrast, when the anomaly degree becomes larger --- increasing $\sigma$ and $\alpha$, the ratio of spectral energy for $\lambda > 0.5$ becomes larger in the histogram. Also, the low-frequency energy ratio curve is almost monotonically decreasing on the interval $\lambda \in [0,2]$. The observation also holds for the Minnesota road graph. 

Moreover, we further assume the features of all nodes are drawn from a single Gaussian distribution which strictly follows Proposition \ref{prop:1}. We corroborate our theoretical findings on three cases: (1) the original features $\mu=1,\sigma=1$, (2) features with a smaller mean value $\mu=0.5,\sigma=1$, and (3) features with a larger variance $\mu=1,\sigma=2$. In Figure \ref{fig:analysis}, both increasing the variance and decreasing the mean value lead to the `right-shift' of spectral energy distributions, which is consistent with the result in Proposition \ref{prop:1}.

\subsection{Validation on Real-world Anomalies}\label{sec:analysis4}

\begin{figure}[t]

\centering
  \includegraphics[width=1.0\columnwidth]{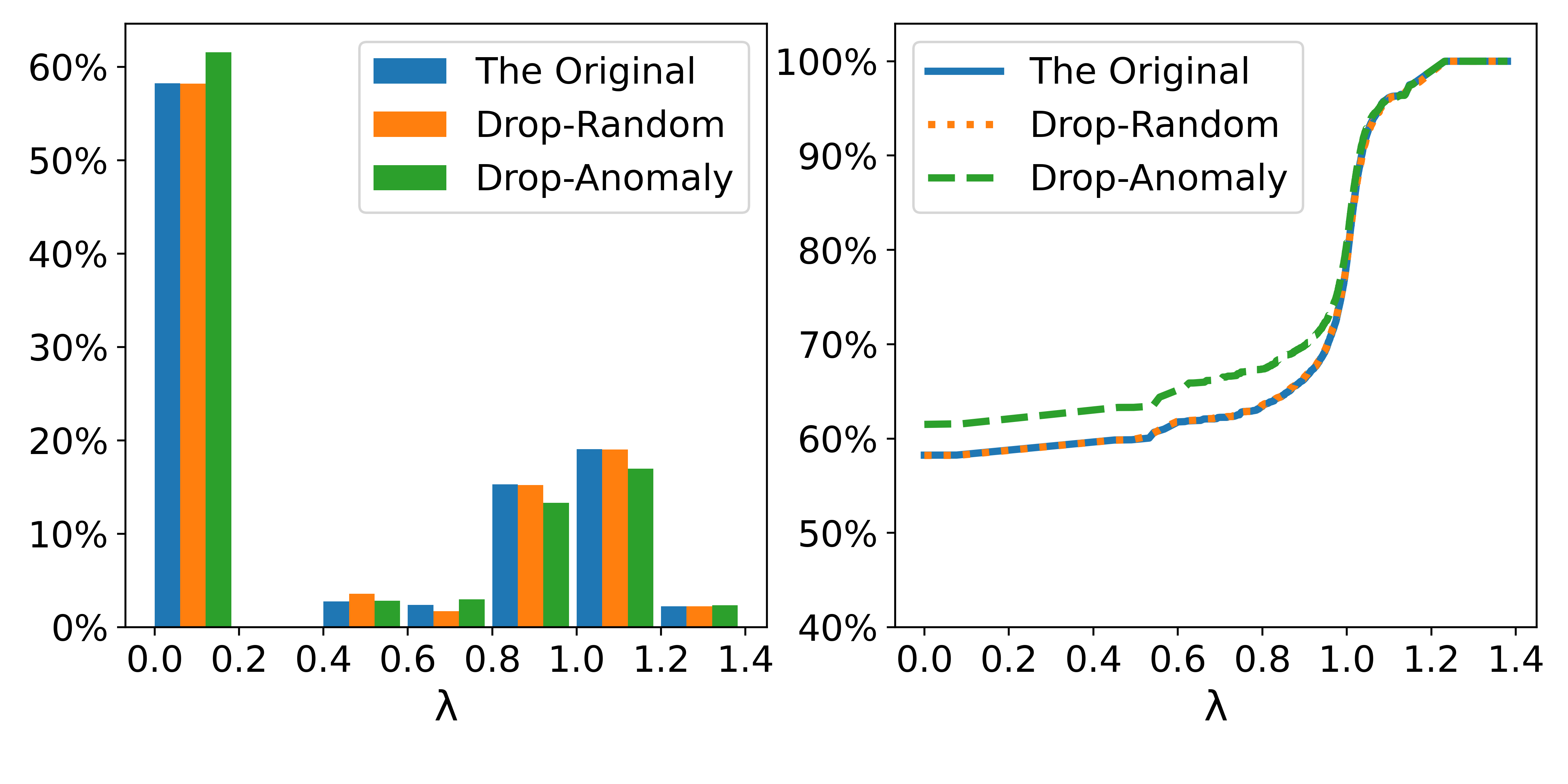}
\vspace{-9mm}
\caption{
Comparison of the spectral energy distribution (left) and the energy ratio curve (right) between the original graph and two perturbed graphs in the Amazon dataset.
}
\label{fig:dataset}
\end{figure}

In real-world anomaly detection datasets, the node features may not strictly follow the Gaussian distribution. Nevertheless, we verify the generality of `right-shift' on the real case. We introduce four industry datasets for different anomaly detection scenarios.  We choose two widely used datasets in previous works \cite{liu2021pick, liu2020alleviating}, including the \textbf{Amazon} dataset~\cite{mcauley2013amateurs} for user anomaly detection and the \textbf{YelpChi} dataset~\cite{rayana2015collective} for review anomaly detection.  We further construct two large-scale real-world datasets as new graph anomaly detection benchmarks, including the \textbf{T-Finance} dataset based on a transaction network and the \textbf{T-Social} dataset based on a social network. The statistics of these datasets are summarized in Table \ref{tab:data}. More details about these datasets are in Section \ref{sec:exp1}.

We first study the `right-shift' phenomenon of anomalies in the Amazon dataset. We choose the last feature dimension and visualize the spectral energy in three views: (1) the origin graph, (2) the perturbed graph by dropping all anomalies, and (3) the perturbed graph by dropping the same number of random nodes. In Figure \ref{fig:dataset}, compared with the origin graph, the spectral energy in the subgraph without anomalies concentrates more in low frequencies (i.e., $\lambda \in [0,0.2)$) and less in high frequencies (i.e., $\lambda \in [0.8,1.2)$). By contrast, dropping random nodes does not have the same effect. The `right-shift' phenomenon in this case indicates that the frequency band around $\lambda=1$ has a strong connection with anomalies.

As the eigen-decomposition operation suffers from the heavy computational burden, especially on the other three large-scale datasets, we quantify the spectral effect by the high-frequency area introduced in Definition \ref{def:2}. We compare the original graph with two perturbed graphs. The relative change is calculated by $\Delta S_{\text{high}} = (\hat S_{\text{high}}-S_{\text{high}})/S_{\text{high}}$, in which $S_{\text{high}}$ and  $\hat S_{\text{high}}$ denote the high-frequency area of the original graph and that of two perturbed graphs respectively.

We compute $\Delta S_{\text{high}}$ for each feature dimension and report the mean and lowest values in Table \ref{tab:data}. Among all datasets, removing anomalies leads to the decrease of $S_{\text{high}}$, while dropping nodes randomly has limited effects on $S_{\text{high}}$. More specifically, on T-Social and Yelpchi, $S_{\text{high}}$ decreases more than 10\% for the worst-case feature, which is impressive given that anomalies are few.

\section{Methodology}\label{sec:method}

The analysis in Section \ref{sec:analysis} shows that we need to focus on `right-shift' effect when detecting graph anomalies. Unfortunately, most of the current GNNs are low-pass filters~\cite{nt2019revisiting, wu2019simplifying} or adaptive filters \cite{defferrard2016convolutional,he2021bernnet,dong2021adagnn,levie2018cayleynets} that are neither guaranteed to be band-pass nor spectral-localized \cite{balcilar2020analyzing}. Since high-frequency anomalies account for only a small fraction and most spectral energy still concentrates on low frequencies, these adaptive GNNs may degrade to low-pass filters in this task. 

To overcome this drawback, we propose our new graph neural network architecture based on Hammond's graph wavelet theory~\cite{hammond2011wavelets}, which is band-pass in nature and can better address the `right-shift' effect inheriting from anomalies. We first introduce the backgrounds of graph wavelet in Section \ref{sec:method1}. Then, we propose our graph Beta wavelet and demonstrate its good properties in Section \ref{sec:method2}. Based on Beta wavelets, we construct the Beta Wavelet Graph Neural Network (BWGNN) in Section \ref{sec:method3}. We discuss the differences between BWGNN and other related graph wavelets in Section \ref{sec:method4}.
\subsection{Background: Hammond's Graph Wavelet}\label{sec:method1}
The graph wavelet transform defined in \cite{hammond2011wavelets} starts with a "mother" wavelet $\psi$ and employs a group of wavelets as bases, defined as $\mathcal W=(\mathcal W_{\psi_1}, \mathcal W_{\psi_2}, \cdots)$. Formally, applying $\mathcal W_{\psi_i}$ on a graph signal $x \in \mathbb R^N$ can be written as
\begin{equation}\label{eq:method_def}
    \mathcal W_{\psi_i}(\bm x) = \bm Ug_i(\bm \Lambda)\bm U^T \bm x,
\end{equation}
where $g_i(\cdot)$ is a kernel function in the spectral domain defined on $[0, \lambda_N]$, and $g_i(\bm \Lambda)=\text{diag}(g_i(\lambda))$. Although Equation \eqref{eq:method_def} is similar to the graph spectral convolution derived from Fourier transform, the kernel function $g_i$ in Hammond's graph wavelet transform should satisfy the following additional requirements.

\begin{itemize}[itemsep=2pt,topsep=0pt,parsep=0pt]
    \item According to the Parseval theorem, the wavelet transform needs to meet the admissibility condition:
    $$ \small \int_0^\infty {\frac{|g_i(w)|^2}{w}} dw = C_{g} < \infty,$$
    which means $g_i$ should satisfy $g_i(0) = g_i(\infty) = 0$ and perform like a band-pass filter in the spectral domain. 
    \item The wavelet transform covers different frequency bands via band-pass filters of different scales $\{g_1, g_2, ...\}$. 
\end{itemize}
To avoid the eigen-decomposition of the graph Laplacian $\bm L$, the kernel function $g_i$ has to be a polynomial function, i.e., $\bm Ug_i(\Lambda)\bm U^T = g_i(\bm L)$ in most of the literature.

\subsection{Beta Wavelet on Graph}\label{sec:method2}

Beta distribution often serves as a wavelet basis \cite{de2015compactly} in computer vision applications \cite{amar2005beta, jemai2010fbwn, eladel2016fast}, but has not been utilized for mining graph data yet. Here we choose the Beta distribution as the graph kernel function and demonstrate that it meets the requirements of Hammond's graph wavelet.

The probability density function of Beta distribution admits:
\begin{equation*}
\beta_{p,q}(w)=\left\{
\begin{array}{ll}
\frac{1}{B(p+1, q+1)}w^p(1-w)^q & \text{if} \ \  w \in [0,1]\\
0 & \text{otherwise}\\
\end{array}\right.
\end{equation*}
where $p, q \in \mathbb R^+$ and $B(p+1,q+1)=p!q!/(p+q+1)!$ is a constant. As the eigenvalues of the normalized graph Laplacian $L$ satisfy $\lambda \in [0,2]$, we adopt $\beta_{p,q}^*(w)=\frac{1}{2}\beta_{p,q}(\frac{w}{2})$ to cover the complete spectral range of $L$. We further add the restrictions $p, q \in \mathbb N^+$ to ensure $\beta^*(p,q)$ is a polynomial, such that fast computation can be conducted. Thus, our Beta wavelet transform $\mathcal W_{p,q}$ can be written as: 
\begin{equation*}
\mathcal W_{p,q}=\bm U\beta^*_{p,q}(\bm \Lambda)\bm U^T=\beta^*_{p,q}(\bm L)=\frac{(\frac{\bm L}{2})^p(I-\frac{\bm L}{2})^q}{2B(p+1,q+1)}.
\end{equation*}
Let $p+q=C$ be a constant and our Beta wavelet transform $\mathcal W$ is constructed by a group of $C+1$ Beta wavelets with the same order: 
\begin{equation}\label{eq:method_wavelets}
    \mathcal W = (\mathcal W_{0,C}, \mathcal W_{1,C-1}, \cdots, \mathcal W_{C,0}).
\end{equation}
In this equation, $\mathcal W_{0,C}$ is a low-pass filter and others are band-pass filters of different scales. Besides, when $p>0$, the kernel function $\beta^*_{p,q}$ satisfies: 
\begin{equation*}
\int_0^\infty {\frac{|\beta^*_{p,q}(w)|^2}{w}} dw \leq \int_0^2\frac{dw}{2B(p+1, q+1)}<\infty.
\end{equation*}
Therefore, the proposed Beta wavelet transform $\mathcal W$ satisfies both two requirements of Hammond's graph wavelet in Section \ref{sec:analysis1}.

\begin{figure}[t]
\centering
  \includegraphics[width=1.0\linewidth]{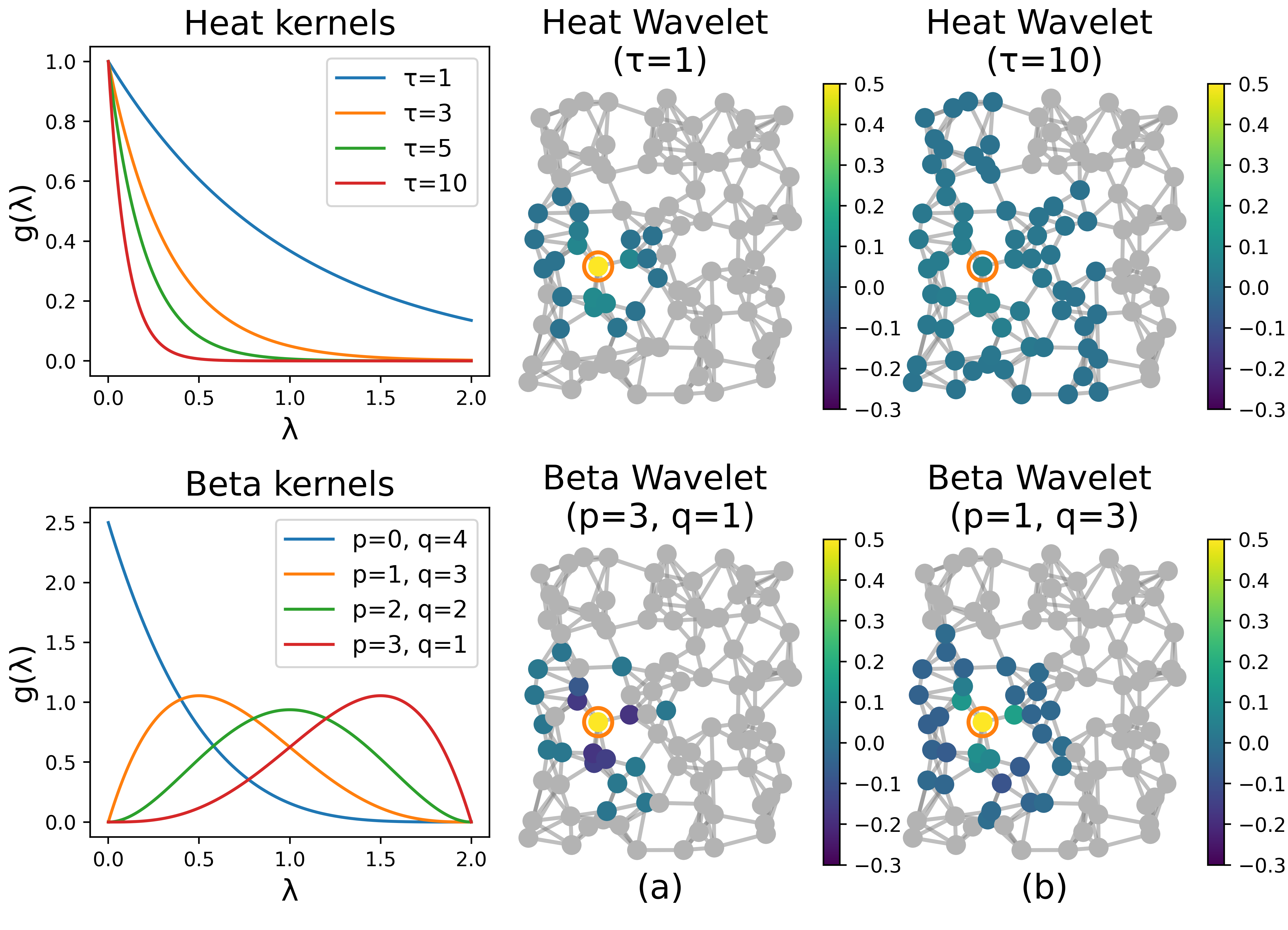}
  \vspace{-7mm}
\caption{Comparison between Heat wavelets and Beta wavelets in the spectral domain (left) and the spatial domain (right).}
\label{fig:model1}
\end{figure}

Figure \ref{fig:model1} compares a group of the proposed Beta wavelets ($C=4$) with the widely-used wavelet with Heat kernels (i.e., $g(\lambda)=e^{-\tau \lambda}, \tau \in \{1,3,5,10\}$). The first column shows that Heat kernels are all low-pass filters in the spectral domain, while Beta kernels contain \emph{various filter types} including low-pass and band-pass. \eat{For any given frequency band, a certain function can be found in the kernel banks to capture corresponding signals. }

The right two columns in Figure \ref{fig:model1} visualize the effect of graph wavelet transform on a randomly generated kNN graph. The responses with an absolute value less than 0.001 are marked as grey. For different scales, responses of heat wavelets are all positive, while those of Beta wavelets can be positive or negative in different channels. The diversified responses of the Beta graph wavelet transform show that it can capture not only the similarities but also the differences of nodes within a subgraph, thus a better \emph{neighborhood flexibility}. This property can make the representations of anomalies more distinguishable.

The classic wavelet transform is well-known for its locality in time and frequency domains. As an analogy, we present two propositions about the \emph{good locality} of Beta graph wavelets in spectral and spatial domains:

\begin{proposition}[Spectral Locality] \label{prop:3}
Consider $p > 0$, $q > 0$ and $X\sim \beta^*_{p,q}$ where $\beta^*_{p,q}$ is band-pass, the mean and variance of $X$ are:
\begin{align*}
    \mu=&\mathbb{E}(X)=\frac{2(p+1)}{p+q+2}\\
    \sigma=&\textnormal{Var}(X)=\frac{4(p+1)(q+1)}{(p+q+2)^2(p+q+3)}.
\end{align*}
When  $p+q \rightarrow \infty$ satisfies $p=cq$, we have $\sigma \rightarrow 0$ and $\mu=\frac{2c}{c+1}$ can be any number between $(0,2)$. 
\end{proposition}
Here, we use $p=cq$ to ensure $p$ and $q$ do not deviate dramatically. Proposition \ref{prop:3} can be directly derived from the properties of the Beta distribution. It shows $\beta^*_{p,q}$ can concentrate on any $\mu \in (0,2)$, which means Beta graph wavelets can be tailored to any specific frequency band for anomaly detection. 

\begin{proposition}[Spatial Locality] \label{prop:4}
Let $v_i, v_j$ be two nodes on $\mathcal G$, $\mathcal W_{p,q}\delta_i[j]$ be the effect of a one-hot signal $\delta_i \in \mathbb R^N$ on node $v_j$ after the wavelet transform. $\mathcal W_{p,q}\delta_i$ is localized in ($p+q$)-hops of node $v_i$. When the distance $d_{\mathcal{G}}(v_i,v_j)>p+q$, we have $\mathcal W_{p,q}\delta_i[j]=0$.
\end{proposition}
Invoking Lemma 5.2 in \cite{hammond2011wavelets}, we have $\bm L^n\delta_i[j]=0$ when $d_{\mathcal{G}}(v_i,v_j)>n$.
Because $\mathcal W_{p,q}$ is just a $C$-order polynomial of $\bm L$, we complete the proof. Proposition \ref{prop:4} shows that the Beta graph wavelet also has a good spatial locality.  Propositions \ref{prop:3} and \ref{prop:4} indicate that a larger $C$ can provide a better spectral locality at the expense of a worse spatial locality and vice versa.

\subsection{Beta Wavelet Graph Neural Network}\label{sec:method3}
Based on the introduced Beta graph wavelet, we propose the Beta Wavelet Graph Neural Network (BWGNN) for graph-based anomaly detection.

Different from GNN models~\cite{Kipf} in which multiple layers are stacked in a cascade fashion, BWGNN uses different wavelet kernels in parallel and then aggregates the corresponding filtering results. Specifically, BWGNN adopts the following propagation process:
\begin{align*}
    \bm Z_i &= \mathcal W_{i,C-i}(\text{MLP}(\bm X))\\
    \bm H &= \text{AGG}([\bm Z_0, \bm Z_1, \cdots, \bm Z_C]),
\end{align*}
where $\text{MLP}(\cdot)$ denotes a multi-layer perceptron  and $\text{AGG}(\cdot)$ can be a simple aggregation function such as summation or concatenation. $\mathcal W_{i,C-i}$ from \eqref{eq:method_wavelets} denotes our wavelet kernels. The aggregated representation $\bm H$ is then fed to another MLP with the Sigmoid function to compute the abnormal probability $p_i$. Weighted cross-entropy loss is used for the training of BWGNN:
\begin{align}\label{eq:method_loss}
    \mathcal L&=\sum_i (\gamma y_i\log(p_i) + (1- y_i)\log(1-p_i)),
\end{align}
where $\gamma$ is the ratio of anomaly labels ($y_i=1$) to normal labels ($y_i=0$).

The complexity of BWGNN is $O(C|\mathcal{E}|)$, as $\beta^*_{p,q}(L)$ is a polynomial function that can be computed recursively \cite{defferrard2016convolutional}.

\subsection{Discussion}\label{sec:method4}

We discuss the differences between the proposed BWGNN and other related wavelet GNNs. Heat kernel-based methods \cite{GWNN, wavelet_jure, gdnli} do not essentially satisfy Hammond's graph wavelet theory \cite{hammond2011wavelets} because they are not band-pass. To remedy this issue, BWGNN extends Hammond's graph wavelet to the learnable GNN framework and, for the first time, utilizes Beta distribution as the kernel function to generate graph wavelets. On another front, \cite{Scattering, Diffusion1, Diffusion2} adopt the idea of diffusion wavelets \cite{coifman2006diffusion} to construct graph neural networks.
Compared with them, the proposed Beta kernel can better handle higher frequency anomalies via multiple flexible, localized, and band-pass filters. 

\begin{table*}[t]
\caption{Experimental results of all compared methods on YelpChi and Amazon with 1\% and 40\% training ratios. }
\vskip 0.1in
\centering
  \begin{tabular}{r|cc|cc|cc|cc}
    \toprule
    Dataset& \multicolumn{2}{c|}{YelpChi (1\%)} & \multicolumn{2}{c|}{YelpChi (40\%)} & \multicolumn{2}{c|}{Amazon (1\%)} & \multicolumn{2}{c}{Amazon (40\%)}\\
    
    Metric & F1-macro & AUC & F1-macro & AUC & F1-macro & AUC & F1-macro & AUC\\
    \midrule
    MLP             & 53.90 & 59.83 & 57.57 & 66.52 & 74.68 & 83.62 & 79.17 & 89.80\\
    SVM             & 60.47 & 62.92 & 70.77 & 70.37 & 83.49 & 81.62 & 90.71 & 90.51\\ \midrule
    GCN             & 52.48 & 54.06 & 54.31 & 56.51 & 67.93 & 82.85 & 67.47 & 83.49\\
    ChebyNet        & 63.13 & 73.48 & 65.72 & 78.19 & 85.74 & 87.60 & 91.94 & 94.64\\
    GAT             & 50.27 & 50.95 & 54.64 & 57.20 & 60.84 & 73.45 & 83.18 & 89.90\\
    GIN             & 57.57 & 64.73 & 62.85 & 74.09 & 68.69 & 78.83 & 69.26 & 80.56\\
    GraphSAGE       & 58.41 & 67.58 & 65.49 & 78.31 & 70.78 & 75.37 & 74.17 & 86.95\\
    GWNN            & 59.10 & 67.16 & 65.29 & 75.32 & 87.01 & 85.37 & 91.00 & 93.19\\\midrule
    GraphConsis     & 56.79 & 66.41 & 58.70 & 69.83 & 68.59 & 74.11 & 75.12 & 87.41\\
    CAREGNN         & 62.18 & 75.07 & 63.32 & 76.19 & 68.78 & 88.69 & 86.39 & 90.53\\
    PC-GNN          & 59.82 & 75.47 & 63.00 & 79.87 & 79.86 & \textbf{90.40} & 89.56 & 95.86\\
    \midrule
    BWGNN (Homo)   & 61.15 & 72.01 & 71.00 & 84.03& \textbf{90.92} & 89.45 & \textbf{92.29} & \textbf{98.06}\\ 
    BWGNN (Hetero) & \textbf{67.02} & \textbf{76.95} & \textbf{76.96} & \textbf{90.54} & 83.83 & 86.59 & 91.72 & 97.42\\ 
    \bottomrule
    
\end{tabular}
\label{tab:exp1}
\end{table*}

\vspace{-0.5\baselineskip}
\section{Experiments}

\subsection{Experimental Setup}\label{sec:exp1}

\noindent\textbf{Datasets.}
We conduct experiments on four datasets introduced in Table \ref{tab:data}. The \textbf{YelpChi} dataset~\cite{rayana2015collective} aims to find the anomalous reviews which unjustly promote or demote certain products or businesses on Yelp.com. There are three edge types in the graph, including R-U-R (the reviews posted by the same user), R-S-R (the reviews under the same product with the same star rating), and R-T-R (the reviews under the same product posted in the same month). The \textbf{Amazon} dataset~\cite{mcauley2013amateurs} aims to find the anomalous users paid to write fake product reviews under the Musical Instrument category on Amazon.com. There are also three relations: U-P-U (users reviewing at least one same product), U-S-U (users having at least one same star rating within one week), and U-V-U (users with top-5\% mutual review similarities).

Furthermore, we release two real-world datasets \textbf{T-Finance} and \textbf{T-Social}. The \textbf{T-Finance} dataset aims to find the anomaly accounts in transaction networks. The nodes are unique anonymized accounts with 10-dimension features related to registration days, logging activities and interaction frequency. The edges in the graph represent two accounts that have transaction records. Human experts annotate nodes as anomalies if they fall into categories like fraud, money laundering and online gambling. The \textbf{T-Social} dataset aims to find the anomaly accounts in social networks. It has the same node annotations and features as T-Finance, while two nodes are connected if they maintain the friend relationship for more than three months. The size of T-Social is 100 times larger than that of YelpChi and Amazon.

\noindent\textbf{Metrics.} We choose two widely used metrics to measure the performance of all the methods, namely \textbf{F1-macro} and \textbf{AUC}. \textbf{F1-macro} is the unweighted mean of the F1-score of two classes, which neglects the imbalance ratio between normal and anomaly labels. \textbf{AUC}~\cite{davis2006relationship} is the area under the ROC Curve.

\noindent\textbf{Baselines and Implementation Details.}
We compare BWGNN with three groups of baselines. The first group only considers node features and includes \textbf{MLP} and \textbf{SVM}~\cite{chang2011libsvm}. The second group is general GNN models, including \textbf{GCN}~\cite{Kipf}, \textbf{ChebyNet}~\cite{defferrard2016convolutional}, \textbf{GAT}~\cite{velivckovic2017graph}, \textbf{GIN}~\cite{xu2018powerful}, \textbf{GraphSAGE}~\cite{hamilton2017inductive}, and  \textbf{GWNN}~\cite{GWNN}. The third group is state-of-the-art methods for graph-based anomaly detection, including \textbf{GraphConsis}~\cite{liu2020alleviating}, \textbf{CAREGNN}~\cite{dou2020enhancing} and \textbf{PC-GNN}~\cite{liu2021pick}. For the detailed baseline description, we refer the reader to Appendix \ref{sec:append4}.

Since the graphs in YelpChi and Amazon are multi-relational, we introduce two ways of dealing with heterogeneity in BWGNN. The first way is to treat all types of edges as the same. The second way is to perform graph propagation \eqref{eq:method_wavelets} for each relation  separately  and add a maximum pooling after that. We denote the first way as \textbf{BWGNN (homo)} and the second way as \textbf{BWGNN (hetero)}. 

We train all models except SVM for 100 epochs by Adam optimizer with a learning rate of 0.01, and save the model with the best Macro-F1 in validation. On the YelpChi, Amazon, and T-Finance datasets, the dimension $h$ for representations and hidden states in all models are set to 64, and the order $C$ in BWGNN is 2. We use concatenation as the $\text{AGG}(\cdot)$ function in BWGNN. The training ratio is 40\% in the supervised scenario and 1\% in the semi-supervised scenario, while the remaining data are split by 1:2 for validation and test. On the T-Social dataset, $h$ is set to 64, $C$ is set to 5, the supervised training ratio is 40\%, and the semi-supervised training ratio is 0.01\% (with only 17 labeled anomalies). The ratio of validation and test sets is 1:2. We report the average value and standard deviation of 10 runs on YelpChi and Amazon. On T-Finance and T-Social, we report the average value of 5 runs with different random seeds. Please refer to Appendix \ref{sec:append4} for more implementation details.

\begin{table*}[t]
\caption{Experimental results and the overall training time (seconds) on the T-Finance and T-Social datasets with different training ratios.}
\vskip 0.1in
\centering
  \resizebox{\linewidth}{!}{
  \begin{tabular}{r|cc|ccc|cc|ccc}
    \toprule
    Dataset & \multicolumn{2}{c|}{T-Finance (1\%)} & \multicolumn{3}{c|}{T-Finance (40\%)} & \multicolumn{2}{c|}{T-Social (0.01\%)} & \multicolumn{3}{c}{T-Social (40\%)} \\  
    Metric& F1-macro & AUC & F1-macro & AUC & Time & F1-macro & AUC & F1-macro & AUC & Time \\
    \midrule
    MLP         & 61.00 & 82.93 & 70.57 & 87.15 & 13.32 & 50.03 & 56.35 & 50.35 & 56.96 & 986 \\
    SVM         & 67.69 & 71.47 & 76.23 & 78.16 & 145.11 & 57.69 & 50.06 & - & - & $>$1 day  \\ 
    \midrule
    GCN         & 54.11 & 57.30 & 70.74 & 64.43 & 23.98 & 49.23 & 59.04 & 59.88 & 87.35 & 1294\\
    ChebyNet    & 77.20 & 85.53 & 80.81 & 88.45 & 26.13 & 52.59 & 70.02 & 64.77 & 85.52 & 1711\\
    GAT         & 53.15 & 52.04 & 53.86 & 73.00 & 181.62& 46.25 & 44.35 & 69.01 & 89.06 & 1596\\
    GIN         & 58.25 & 68.86 & 65.23 & 80.02 & 32.39 & 58.32 & 70.61 & 61.74 & 79.72 & 2195\\
    GraphSAGE   & 59.03 & 66.35 & 52.71 & 67.12 & 35.91 & 57.91 & 59.69 & 59.77 & 70.80 & 2230\\
    GWNN        & 70.64 & 86.68 & 71.58 & 86.57 & 27.25 & 50.81 & 56.14 & 58.72 & 73.77 & 1992\\ 
    \midrule
    GraphConsis & 71.73 & 90.28 & 73.46 & 91.42 & 264.41 & 52.45 & 65.29 & 56.55 & 71.25 & 3495\\
    CAREGNN     & 73.32 & 90.50 & 77.55 & 92.16 & 572.41 & 55.82 & 71.20 & 56.26 & 71.86 & 9159\\
    PC-GNN      & 62.06 & 90.76 & 63.18 & 91.23 & 736.55 & 51.14 & 59.84 & 52.17 & 68.45 & 13958\\
    \midrule
    BWGNN       & \textbf{84.89} & \textbf{91.15} & \textbf{86.87} & \textbf{94.35} & 31.98 & \textbf{75.93} & \textbf{88.06} & \textbf{83.98} & \textbf{95.20} & 2707\\
  \bottomrule
\end{tabular}
}
\label{tab:exp2}
\end{table*}

\subsection{Performance Comparison}
\vspace{0.6mm}

The results are reported in Table~\ref{tab:exp1} and Table~\ref{tab:exp2} respectively. The standard deviation is in Appendix \ref{sec:append5}. In general, BWGNN achieves the best performance in all datasets except Amazon (1\%), where PC-GNN obtains the best AUC score. For two datasets with multi-relation graphs, BWGNN (Hetero) performs better on YelpChi while BWGNN (Homo) is better on Amazon. 

GraphConsis, CAREGNN, and PC-GNN are three state-of-the-art methods for graph-based anomaly detection, while BWGNN outperforms them significantly with a much shorter training time. For example, on YelpChi (40\%), BWGNN has 13.9\% and 10.6\% absolutely improvement in F1-Macro and AUC respectively when compared with PC-GNN. The performance improvement is even more significant on T-Social, proving both the superiority and the scalability of BWGNN on large graphs. 

Among general GNN models (GCN, ChebyNet, GAT, GIN, GraphSAGE, and GWNN), GCN performs worst in most cases. It is consistent with our analysis that the low-pass filter is insufficient in distinguishing anomalies. Conversely, ChebyNet performs better because the learnable Chebyshev kernel in it can act as a band-pass filter.

Even though graph structure is ignored, MLP and SVM can also achieve comparable performance in some datasets. For instance, SVM outperforms many GNN approaches on Amazon. However, the performance is still lower than BWGNN, suggesting that neighborhood information is still important in anomaly detection, but needs to be used properly.

\begin{figure}[t]
\centering
  \includegraphics[width=1.0\columnwidth]{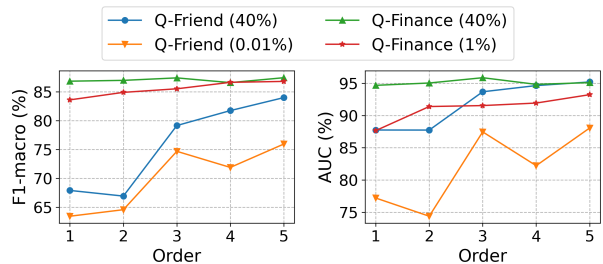}
\vspace{-6mm}
\caption{The performance of BWGNN on T-Finance and T-Social with different order $C$.}
\label{fig:exp1}
\end{figure}

\subsection{Sensitivity Analysis}

\noindent\textbf{The Order $C$ in BWGNN.}
The order $C$ is a crucial hyper-parameter in BWGNN, as Beta wavelet is a $C$-order polynomial of $L$ and is localized in $C$-hops of each node according to proposition~\ref{prop:4}.  Figure~\ref{fig:exp1} presents the F1-macro and AUC scores of BWGNN on two datasets when varying  $C$ from 1 to 5. On T-Social, higher $C$ leads to better performances, while on T-Finance, there are no significant differences in results for $C \geq 2$. One possible reason is that the graph in T-Social is much more sparse than that in T-Finance according to Table~\ref{tab:data}. Thus, a larger range of neighborhoods is required for T-Social.

\begin{figure}[h!]
\centering
  \includegraphics[width=1.0\columnwidth]{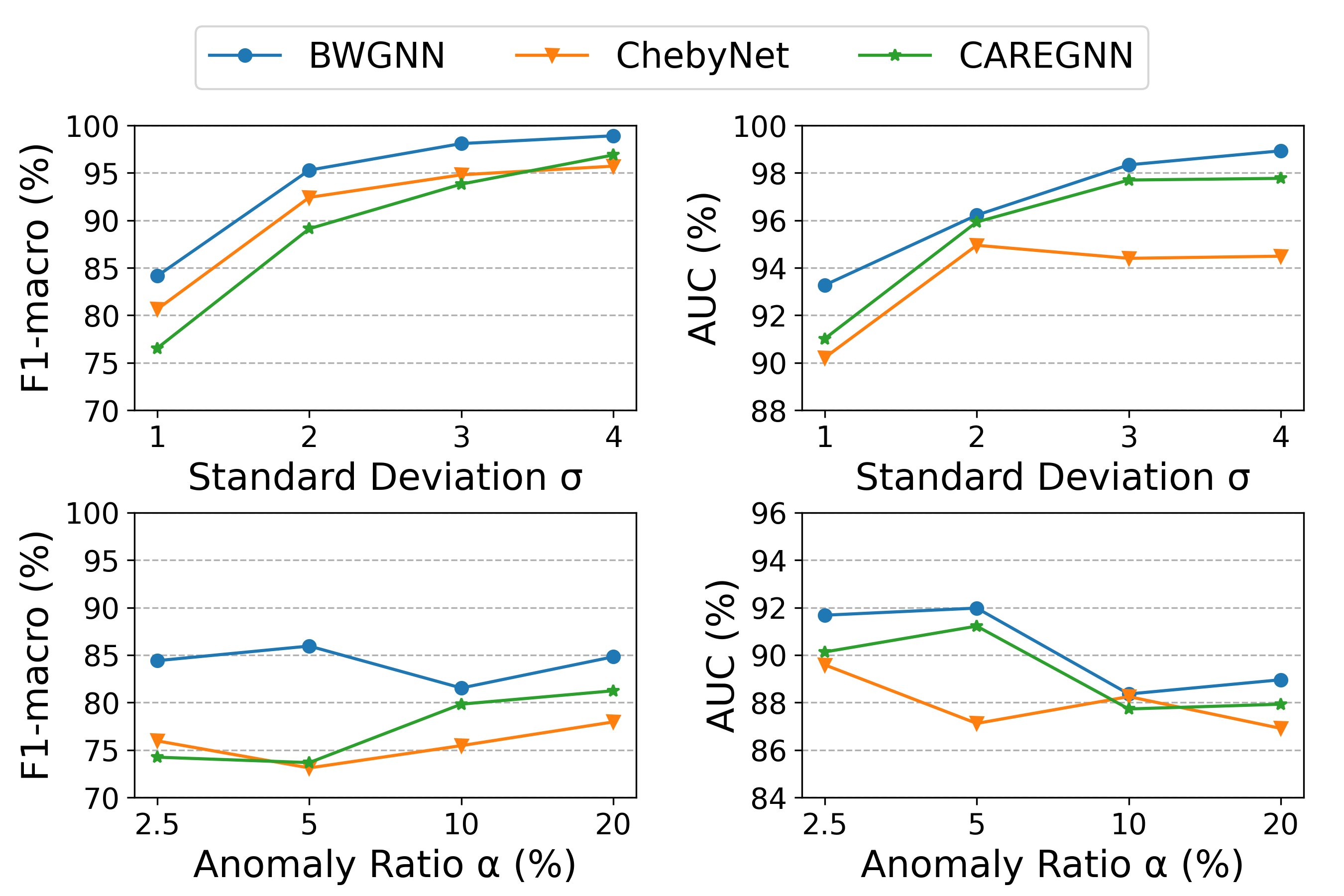}
\vspace{-6mm}
\caption{Comparison of performance on T-Finance (1\%) with different anomaly degrees controlled by $\sigma$ and $\alpha$.}
\vspace{-5mm}

\label{fig:exp2}
\end{figure}

\noindent\textbf{Impact of Anomaly Degree.}
We evaluate the effect of different anomaly degrees on ChebyNet, CAREGNN and our BWGNN using the T-Finance dataset. We consider two variables in Section \ref{sec:analysis2}, including the standard deviation $\sigma$ and the fraction $\alpha$ of the anomalies. We keep the mean value and scale the  $\sigma$ of anomalous node features. To control $\alpha$, we follow the attribute perturbation schema in \cite{dominant}. We generate different fractions of anomalies by replacing normal node attributes with anomalous ones.

Figure~\ref{fig:exp2} compares the F1-macro and AUC scores of BWGNN, ChebyNet, and CAREGNN on T-Finance (1\%) with different anomaly degrees. When $\sigma$ increases, all three models perform better as the anomalies are more distinguishable. Among them, BWGNN is the fastest-growing method and reaches 99\% F1-macro at $\sigma=4$. When varying $\alpha$, BWGNN consistently outperforms other methods and is robust to different anomaly degrees.

\section{Conclusion}
This work presents a novel analysis of graph anomalies in the spectral domain. We find that graph anomalies lead to the `right-shift' phenomenon of spectral energy distributions and further rigorously justify the observation on a vanilla probabilistic model. Inspired by this fact, we propose Beta Wavelet Graph Neural network (BWGNN) to better capture anomaly information on graph. BWGNN leverages Beta graph wavelet to generate band-pass filters with good locality in spectral and spatial domains. Empirical results on four datasets show the superiority and scalability of our model.

\section*{Acknowledgements}

The work described in this paper was supported by grants from HKUST(GZ) under a Startup Grant and HKUST-GZU Joint Research Collaboration Fund (Project No.: GZU22EG05).

\bibliography{ref}

\begin{thebibliography}{56}
\providecommand{\natexlab}[1]{#1}
\providecommand{\url}[1]{\texttt{#1}}
\expandafter\ifx\csname urlstyle\endcsname\relax
  \providecommand{\doi}[1]{doi: #1}\else
  \providecommand{\doi}{doi: \begingroup \urlstyle{rm}\Url}\fi

\bibitem[Amar et~al.(2005)Amar, Zaied, and Alimi]{amar2005beta}
Amar, C.~B., Zaied, M., and Alimi, A.
\newblock Beta wavelets. synthesis and application to lossy image compression.
\newblock \emph{Advances in Engineering Software}, 36\penalty0 (7):\penalty0
  459--474, 2005.

\bibitem[Balcilar et~al.(2020)Balcilar, Renton, H{\'e}roux, Ga{\"u}z{\`e}re,
  Adam, and Honeine]{balcilar2020analyzing}
Balcilar, M., Renton, G., H{\'e}roux, P., Ga{\"u}z{\`e}re, B., Adam, S., and
  Honeine, P.
\newblock Analyzing the expressive power of graph neural networks in a spectral
  perspective.
\newblock In \emph{ICLR}, 2020.

\bibitem[Bandyopadhyay et~al.(2019)Bandyopadhyay, Lokesh, and Murty]{ONE}
Bandyopadhyay, S., Lokesh, N., and Murty, M.~N.
\newblock Outlier aware network embedding for attributed networks.
\newblock In \emph{{AAAI}}, 2019.

\bibitem[Bao et~al.(2019)Bao, Tang, Li, and Zhang]{anomaly_health}
Bao, Y., Tang, Z., Li, H., and Zhang, Y.
\newblock Computer vision and deep learning--based data anomaly detection
  method for structural health monitoring.
\newblock \emph{Structural Health Monitoring}, 18\penalty0 (2):\penalty0
  401--421, 2019.

\bibitem[Bedeian \& Mossholder(2000)Bedeian and Mossholder]{bedeian2000use}
Bedeian, A.~G. and Mossholder, K.~W.
\newblock On the use of the coefficient of variation as a measure of diversity.
\newblock \emph{Organizational Research Methods}, 3\penalty0 (3):\penalty0
  285--297, 2000.

\bibitem[Bo et~al.(2021)Bo, Wang, Shi, and Shen]{bo2021beyond}
Bo, D., Wang, X., Shi, C., and Shen, H.
\newblock Beyond low-frequency information in graph convolutional networks.
\newblock In \emph{AAAI}, 2021.

\bibitem[Chang \& Lin(2011)Chang and Lin]{chang2011libsvm}
Chang, C.-C. and Lin, C.-J.
\newblock Libsvm: a library for support vector machines.
\newblock \emph{ACM transactions on intelligent systems and technology (TIST)},
  2\penalty0 (3):\penalty0 1--27, 2011.

\bibitem[Coifman \& Maggioni(2006)Coifman and Maggioni]{coifman2006diffusion}
Coifman, R.~R. and Maggioni, M.
\newblock Diffusion wavelets.
\newblock \emph{Applied and computational harmonic analysis}, 21\penalty0
  (1):\penalty0 53--94, 2006.

\bibitem[Cui et~al.(2020)Cui, Seo, Tabar, Ma, Wang, and Lee]{cui2020deterrent}
Cui, L., Seo, H., Tabar, M., Ma, F., Wang, S., and Lee, D.
\newblock Deterrent: Knowledge guided graph attention network for detecting
  healthcare misinformation.
\newblock In \emph{KDD}, pp.\  492--502, 2020.

\bibitem[Davis \& Goadrich(2006)Davis and Goadrich]{davis2006relationship}
Davis, J. and Goadrich, M.
\newblock The relationship between precision-recall and roc curves.
\newblock In \emph{ICML}, pp.\  233--240, 2006.

\bibitem[De~Oliveira \& De~Ara{\'u}jo(2015)De~Oliveira and
  De~Ara{\'u}jo]{de2015compactly}
De~Oliveira, H. and De~Ara{\'u}jo, G.
\newblock Compactly supported one-cyclic wavelets derived from beta
  distributions.
\newblock \emph{arXiv:1502.02166}, 2015.

\bibitem[Defferrard et~al.(2016)Defferrard, Bresson, and
  Vandergheynst]{defferrard2016convolutional}
Defferrard, M., Bresson, X., and Vandergheynst, P.
\newblock Convolutional neural networks on graphs with fast localized spectral
  filtering.
\newblock \emph{NeurIPS}, pp.\  3844--3852, 2016.

\bibitem[Ding et~al.(2019)Ding, Li, Bhanushali, and Liu]{dominant}
Ding, K., Li, J., Bhanushali, R., and Liu, H.
\newblock Deep anomaly detection on attributed networks.
\newblock In \emph{ICDM}. SIAM, 2019.

\bibitem[Dong et~al.(2021)Dong, Ding, Jalaian, Ji, and Li]{dong2021adagnn}
Dong, Y., Ding, K., Jalaian, B., Ji, S., and Li, J.
\newblock Adagnn: Graph neural networks with adaptive frequency response
  filter.
\newblock In \emph{CIKM}, pp.\  392--401, 2021.

\bibitem[Donnat et~al.(2018)Donnat, Zitnik, Hallac, and Leskovec]{wavelet_jure}
Donnat, C., Zitnik, M., Hallac, D., and Leskovec, J.
\newblock Learning structural node embeddings via diffusion wavelets.
\newblock In \emph{KDD}, 2018.

\bibitem[Dou et~al.(2020)Dou, Liu, Sun, Deng, Peng, and Yu]{dou2020enhancing}
Dou, Y., Liu, Z., Sun, L., Deng, Y., Peng, H., and Yu, P.~S.
\newblock Enhancing graph neural network-based fraud detectors against
  camouflaged fraudsters.
\newblock In \emph{CIKM}, pp.\  315--324, 2020.

\bibitem[ElAdel et~al.(2016)ElAdel, Zaied, and Amar]{eladel2016fast}
ElAdel, A., Zaied, M., and Amar, C.~B.
\newblock Fast beta wavelet network-based feature extraction for image copy
  detection.
\newblock \emph{Neurocomputing}, 173:\penalty0 306--316, 2016.

\bibitem[Gama et~al.(2019)Gama, Ribeiro, and Bruna]{Diffusion1}
Gama, F., Ribeiro, A., and Bruna, J.
\newblock Diffusion scattering transforms on graphs.
\newblock In \emph{{ICLR}}, 2019.

\bibitem[Grubbs(1969)]{grubbs1969procedures}
Grubbs, F.~E.
\newblock Procedures for detecting outlying observations in samples.
\newblock \emph{Technometrics}, 11\penalty0 (1):\penalty0 1--21, 1969.

\bibitem[Hamilton et~al.(2017)Hamilton, Ying, and
  Leskovec]{hamilton2017inductive}
Hamilton, W.~L., Ying, R., and Leskovec, J.
\newblock Inductive representation learning on large graphs.
\newblock In \emph{NeurIPS}, pp.\  1025--1035, 2017.

\bibitem[Hammond et~al.(2011)Hammond, Vandergheynst, and
  Gribonval]{hammond2011wavelets}
Hammond, D.~K., Vandergheynst, P., and Gribonval, R.
\newblock Wavelets on graphs via spectral graph theory.
\newblock \emph{Applied and Computational Harmonic Analysis}, 30\penalty0
  (2):\penalty0 129--150, 2011.

\bibitem[Han et~al.(2011)Han, Kamber, and Pei]{textbook}
Han, J., Kamber, M., and Pei, J.
\newblock \emph{Data Mining: Concepts and Techniques, 3rd edition}.
\newblock Morgan Kaufmann, 2011.

\bibitem[He et~al.(2021)He, Wei, Huang, and Xu]{he2021bernnet}
He, M., Wei, Z., Huang, Z., and Xu, H.
\newblock Bernnet: Learning arbitrary graph spectral filters via bernstein
  approximation.
\newblock \emph{NeurIPS}, 2021.

\bibitem[Jemai et~al.(2010)Jemai, Zaied, Amar, and Alimi]{jemai2010fbwn}
Jemai, O., Zaied, M., Amar, C.~B., and Alimi, A.~M.
\newblock Fbwn: An architecture of fast beta wavelet networks for image
  classification.
\newblock In \emph{{IJCNN}}, pp.\  1--8. IEEE, 2010.

\bibitem[Kendall et~al.(1946)]{kendall1946advanced}
Kendall, M.~G. et~al.
\newblock The advanced theory of statistics.
\newblock \emph{The advanced theory of statistics.}, 1946.

\bibitem[Kipf \& Welling(2017)Kipf and Welling]{Kipf}
Kipf, T.~N. and Welling, M.
\newblock Semi-supervised classification with graph convolutional networks.
\newblock In \emph{{ICLR}}, 2017.

\bibitem[Kumar et~al.(2018)Kumar, Hooi, Makhija, Kumar, Faloutsos, and
  Subrahmanian]{kumar2018rev2}
Kumar, S., Hooi, B., Makhija, D., Kumar, M., Faloutsos, C., and Subrahmanian,
  V.
\newblock Rev2: Fraudulent user prediction in rating platforms.
\newblock In \emph{WSDM}, pp.\  333--341, 2018.

\bibitem[Levie et~al.(2018)Levie, Monti, Bresson, and
  Bronstein]{levie2018cayleynets}
Levie, R., Monti, F., Bresson, X., and Bronstein, M.~M.
\newblock Cayleynets: Graph convolutional neural networks with complex rational
  spectral filters.
\newblock \emph{IEEE Transactions on Signal Processing}, 67\penalty0
  (1):\penalty0 97--109, 2018.

\bibitem[Li et~al.(2021)Li, Li, Liu, Yu, Li, and Cheng]{gdnli}
Li, J., Li, J., Liu, Y., Yu, J., Li, Y., and Cheng, H.
\newblock Deconvolutional networks on graph data.
\newblock In \emph{NeurIPS}, volume~34, pp.\  21019--21030, 2021.

\bibitem[Li et~al.(2018)Li, Han, and Wu]{li2018in}
Li, Q., Han, Z., and Wu, X.
\newblock Deeper insights into graph convolutional networks for semi-supervised
  learning.
\newblock In \emph{AAAI}, pp.\  3538--3545, 2018.

\bibitem[Liu et~al.(2021{\natexlab{a}})Liu, Sun, Ao, Feng, He, and
  Yang]{liu2021intention}
Liu, C., Sun, L., Ao, X., Feng, J., He, Q., and Yang, H.
\newblock Intention-aware heterogeneous graph attention networks for fraud
  transactions detection.
\newblock In \emph{KDD}, pp.\  3280--3288, 2021{\natexlab{a}}.

\bibitem[Liu et~al.(2021{\natexlab{b}})Liu, Ao, Qin, Chi, Feng, Yang, and
  He]{liu2021pick}
Liu, Y., Ao, X., Qin, Z., Chi, J., Feng, J., Yang, H., and He, Q.
\newblock Pick and choose: A gnn-based imbalanced learning approach for fraud
  detection.
\newblock In \emph{Proceedings of the Web Conference 2021}, pp.\  3168--3177,
  2021{\natexlab{b}}.

\bibitem[Liu et~al.(2020)Liu, Dou, Yu, Deng, and Peng]{liu2020alleviating}
Liu, Z., Dou, Y., Yu, P.~S., Deng, Y., and Peng, H.
\newblock Alleviating the inconsistency problem of applying graph neural
  network to fraud detection.
\newblock In \emph{SIGIR}, pp.\  1569--1572, 2020.

\bibitem[Ma et~al.(2021)Ma, Wu, Xue, Yang, Zhou, Sheng, Xiong, and
  Akoglu]{anomaly_survey}
Ma, X., Wu, J., Xue, S., Yang, J., Zhou, C., Sheng, Q.~Z., Xiong, H., and
  Akoglu, L.
\newblock A comprehensive survey on graph anomaly detection with deep learning.
\newblock \emph{IEEE Transactions on Knowledge and Data Engineering}, 2021.

\bibitem[McAuley \& Leskovec(2013)McAuley and Leskovec]{mcauley2013amateurs}
McAuley, J.~J. and Leskovec, J.
\newblock From amateurs to connoisseurs: modeling the evolution of user
  expertise through online reviews.
\newblock In \emph{WWW}, pp.\  897--908, 2013.

\bibitem[Min et~al.(2020)Min, Wenkel, and Wolf]{Scattering}
Min, Y., Wenkel, F., and Wolf, G.
\newblock Scattering {GCN:} overcoming oversmoothness in graph convolutional
  networks.
\newblock In \emph{NeurIPS}, 2020.

\bibitem[Min et~al.(2021)Min, Wenkel, and Wolf]{Diffusion2}
Min, Y., Wenkel, F., and Wolf, G.
\newblock Geometric scattering attention networks.
\newblock In \emph{{ICASSP}}, 2021.

\bibitem[Ngai et~al.(2011)Ngai, Hu, Wong, Chen, and Sun]{anomaly_financial}
Ngai, E.~W., Hu, Y., Wong, Y.~H., Chen, Y., and Sun, X.
\newblock The application of data mining techniques in financial fraud
  detection: A classification framework and an academic review of literature.
\newblock \emph{Decision support systems}, 50\penalty0 (3):\penalty0 559--569,
  2011.

\bibitem[Noble \& Cook(2003)Noble and Cook]{noble2003graph}
Noble, C.~C. and Cook, D.~J.
\newblock Graph-based anomaly detection.
\newblock In \emph{KDD}, pp.\  631--636, 2003.

\bibitem[Nt \& Maehara(2019)Nt and Maehara]{nt2019revisiting}
Nt, H. and Maehara, T.
\newblock Revisiting graph neural networks: All we have is low-pass filters.
\newblock \emph{arXiv:1905.09550}, 2019.

\bibitem[Paszke et~al.(2019)Paszke, Gross, Massa, Lerer, Bradbury, Chanan,
  Killeen, Lin, Gimelshein, Antiga, et~al.]{paszke2019pytorch}
Paszke, A., Gross, S., Massa, F., Lerer, A., Bradbury, J., Chanan, G., Killeen,
  T., Lin, Z., Gimelshein, N., Antiga, L., et~al.
\newblock Pytorch: An imperative style, high-performance deep learning library.
\newblock \emph{NeurIPS}, 32:\penalty0 8026--8037, 2019.

\bibitem[Pedregosa et~al.(2011)Pedregosa, Varoquaux, Gramfort, Michel, Thirion,
  Grisel, Blondel, Prettenhofer, Weiss, Dubourg, et~al.]{pedregosa2011scikit}
Pedregosa, F., Varoquaux, G., Gramfort, A., Michel, V., Thirion, B., Grisel,
  O., Blondel, M., Prettenhofer, P., Weiss, R., Dubourg, V., et~al.
\newblock Scikit-learn: Machine learning in python.
\newblock \emph{the Journal of machine Learning research}, 12:\penalty0
  2825--2830, 2011.

\bibitem[Perraudin et~al.(2014)Perraudin, Paratte, Shuman, Martin, Kalofolias,
  Vandergheynst, and Hammond]{perraudin2014gspbox}
Perraudin, N., Paratte, J., Shuman, D., Martin, L., Kalofolias, V.,
  Vandergheynst, P., and Hammond, D.~K.
\newblock Gspbox: A toolbox for signal processing on graphs.
\newblock \emph{arXiv preprint arXiv:1408.5781}, 2014.

\bibitem[Rayana \& Akoglu(2015)Rayana and Akoglu]{rayana2015collective}
Rayana, S. and Akoglu, L.
\newblock Collective opinion spam detection: Bridging review networks and
  metadata.
\newblock In \emph{KDD}, pp.\  985--994, 2015.

\bibitem[Sipple(2020)]{sipple2020interpretable}
Sipple, J.
\newblock Interpretable, multidimensional, multimodal anomaly detection with
  negative sampling for detection of device failure.
\newblock In \emph{ICML}, pp.\  9016--9025. PMLR, 2020.

\bibitem[Spielman(2007)]{spielman2007spectral}
Spielman, D.~A.
\newblock Spectral graph theory and its applications.
\newblock In \emph{48th Annual IEEE Symposium on Foundations of Computer
  Science (FOCS'07)}, pp.\  29--38. IEEE, 2007.

\bibitem[Ten et~al.(2011)Ten, Hong, and Liu]{anomaly_cyber}
Ten, C.-W., Hong, J., and Liu, C.-C.
\newblock Anomaly detection for cybersecurity of the substations.
\newblock \emph{IEEE Transactions on Smart Grid}, 2\penalty0 (4):\penalty0
  865--873, 2011.

\bibitem[Veli{\v{c}}kovi{\'c} et~al.(2017)Veli{\v{c}}kovi{\'c}, Cucurull,
  Casanova, Romero, Lio, and Bengio]{velivckovic2017graph}
Veli{\v{c}}kovi{\'c}, P., Cucurull, G., Casanova, A., Romero, A., Lio, P., and
  Bengio, Y.
\newblock Graph attention networks.
\newblock \emph{arXiv:1710.10903}, 2017.

\bibitem[Wang et~al.(2019{\natexlab{a}})Wang, Lin, Cui, Jia, Wang, Fang, Yu,
  Zhou, Yang, and Qi]{wang2019semi}
Wang, D., Lin, J., Cui, P., Jia, Q., Wang, Z., Fang, Y., Yu, Q., Zhou, J.,
  Yang, S., and Qi, Y.
\newblock A semi-supervised graph attentive network for financial fraud
  detection.
\newblock In \emph{ICDM}, pp.\  598--607. IEEE, 2019{\natexlab{a}}.

\bibitem[Wang et~al.(2019{\natexlab{b}})Wang, Zheng, Ye, Gan, Li, Song, Zhou,
  Ma, Yu, Gai, Xiao, He, Karypis, Li, and Zhang]{wang2019dgl}
Wang, M., Zheng, D., Ye, Z., Gan, Q., Li, M., Song, X., Zhou, J., Ma, C., Yu,
  L., Gai, Y., Xiao, T., He, T., Karypis, G., Li, J., and Zhang, Z.
\newblock Deep graph library: A graph-centric, highly-performant package for
  graph neural networks.
\newblock \emph{arXiv:1909.01315}, 2019{\natexlab{b}}.

\bibitem[Wu et~al.(2019)Wu, Souza, Zhang, Fifty, Yu, and
  Weinberger]{wu2019simplifying}
Wu, F., Souza, A., Zhang, T., Fifty, C., Yu, T., and Weinberger, K.
\newblock Simplifying graph convolutional networks.
\newblock In \emph{ICML}, pp.\  6861--6871, 2019.

\bibitem[Wu et~al.(2021)Wu, Pan, Long, Jiang, and Zhang]{wu2021beyond}
Wu, Z., Pan, S., Long, G., Jiang, J., and Zhang, C.
\newblock Beyond low-pass filtering: Graph convolutional networks with
  automatic filtering.
\newblock \emph{arXiv preprint arXiv:2107.04755}, 2021.

\bibitem[Xu et~al.(2019{\natexlab{a}})Xu, Shen, Cao, Qiu, and Cheng]{GWNN}
Xu, B., Shen, H., Cao, Q., Qiu, Y., and Cheng, X.
\newblock Graph wavelet neural network.
\newblock In \emph{{ICLR}}, 2019{\natexlab{a}}.

\bibitem[Xu et~al.(2019{\natexlab{b}})Xu, Hu, Leskovec, and
  Jegelka]{xu2018powerful}
Xu, K., Hu, W., Leskovec, J., and Jegelka, S.
\newblock How powerful are graph neural networks?
\newblock \emph{ICLR}, 2019{\natexlab{b}}.

\bibitem[Zhao et~al.(2020)Zhao, Deng, Yu, Jiang, Wang, and
  Jiang]{zhao2020error}
Zhao, T., Deng, C., Yu, K., Jiang, T., Wang, D., and Jiang, M.
\newblock Error-bounded graph anomaly loss for gnns.
\newblock In \emph{CIKM}, pp.\  1873--1882, 2020.

\bibitem[Zhao et~al.(2021)Zhao, Jiang, Shah, and Jiang]{zhao2021synergistic}
Zhao, T., Jiang, T., Shah, N., and Jiang, M.
\newblock A synergistic approach for graph anomaly detection with pattern
  mining and feature learning.
\newblock \emph{IEEE Transactions on Neural Networks and Learning Systems},
  2021.

\end{thebibliography}
\bibliographystyle{icml2022}

\appendix

\newpage
\onecolumn 

\icmltitle{Supplementary of ``Rethinking Graph Neural Networks for Anomaly Detection''}
\section{Proof of Proposition 2} \label{sec:append1}

\begin{proof}
By the rotation invariance of Gaussian distributions, 
\[\bm x \sim \mathcal{N}(\mu e_N, \sigma^2 \bm I_N) \Rightarrow \bm{\hat x} = \bm U^T \bm x \sim \mathcal{N}(\mu \bm U^T e_N, \sigma^2 \bm I_N).\] 
As the all-the-one vector is the eigenvector for $\lambda_1 = 0$, we have $\hat x_1 \sim \mathcal{N}(\mu\sqrt{n}, \sigma^2)$ and $\hat x_i \sim \mathcal{N}(0, \sigma^2), \forall i\neq 1$. 
For simplicity, we denote $z_i = \frac{\hat x_i}{\sigma}$ then $z_1 \sim \mathcal{N}(\tfrac{\mu\sqrt{N}}{\sigma}, 1)$ and $z_i \sim \mathcal{N}(0, 1), \forall i\neq 1$. 
\begin{align*}
 \mathbb{E}_{\bm x}\left[\frac{1}{\eta_k(\bm x,\bm L)}\right]-1 &= \mathbb{E}\left[\frac{\sum_{i=k+1}^N \hat x_i^2}{\sum_{i=1}^k \hat x_i^2}\right]\\
 &=\mathbb{E}\left[\frac{\sum_{i=k+1}^N z_i^2}{\sum_{i=1}^k z_i^2}\right]\\
    &= \mathbb{E}\left[\sum_{i=k+1}^N z_i^2\right] \mathbb{E}\left[\frac{1}{\sum_{i=1}^k z_i^2}\right] \\
    & = (N-k) \mathbb{E}\left[\frac{1}{z_i^2 + \sum_{i=2}^k z_i^2}\right].
\end{align*}
Notice that $\sum_{i=2}^k z_i^2  \sim \tilde{\chi}^2(k-1)$ follows the chi-square distribution which is also independent with $z_1$. Denote $ w = \sum_{i=2}^k z_i^2 \ge 0$ and $\rho= \frac{\mu\sqrt{n}}{\sigma}$, we have
\begin{align*}
    \mathbb{E}_{\bm x}\left[\frac{1}{\eta_k(\bm x,\bm L)}\right]-1 & \propto \mathbb{E}_{ w} \left[ \mathbb{E}_{z_1}\left[\frac{1}{z_1^2+w}\right]\right] \\
&=  \mathbb{E}_{ w}\left[\int_{-\infty}^{+\infty} \frac{1}{\sqrt{2 \pi}} e^{-\frac{1}{2}(t-\rho)^{2}} \frac{1}{t^{2}+w} d t\right] \\
&\propto \mathbb{E}_{ w}\left[\int_{-\infty}^{+\infty} e^{-\frac{1}{2}(t-\rho)^{2}} \frac{1}{t^{2}+w} d t \right].
\end{align*}
As $w$ is nonnegative, we just have to focus on the 
monotonicity of $f(\rho) = \int_{-\infty}^{+\infty} e^{-\frac{1}{2}(t-\rho)^{2}} \frac{1}{t^{2}+w} d t$ with respect to $\rho$. Due to Lemma \ref{lm:log-cav}, we can conclude that $f(\rho)$ is log-concave function and the maximum attains at $0$. Hence, the expectation of the inverse of low-frequency energy ratio $\mathbb{E}_{\bm x}[1/\eta_k(\bm x,\bm L)]$ is monotonically increasing with the anomaly degree $ \frac{\sigma}{|\mu|}$. 
\end{proof}

\begin{lemma}\label{lm:log-cav}
If $w$ is nonnegative, $f(\rho) = \int_{-\infty}^{+\infty} e^{-\frac{1}{2}(t-\rho)^{2}} \frac{1}{t^{2}+w} d t$ is a log-concave function and the maximum attains at 0. 
\end{lemma}

\begin{proof}
By the following two log-concave preserving rules, it is not hard to get our result. 
\begin{itemize}
\item The product of log-concave functions is still a log-concave function. \item If $f:\mathbb{R}^n \times \mathbb{R}^n \rightarrow \mathbb{R}$ is a log-concave function, then 
$
g(x) = \int f(x,y) dy
$
is still a log-concave function. 
\end{itemize}
Combining with the fact that the density function of Gaussian distributions is log-concave, it is easy to get $f(\rho)$ is log-concave. We omitted the details here. The next step is to argue the optimal solution. As $f(\rho)$ is differentiable, we can interchange the partial and integral here, that is, 
\[
\frac{\partial f(\rho)}{\partial \rho} = \int_{-\infty}^{+\infty} \frac{1}{t^{2}+w} e^{-\frac{1}{2}(t-\rho)^{2}}(t-\rho) dt. 
\]
Then, 
\begin{align*}
\frac{\partial f(\rho)}{\partial \rho}|_{\rho = 0} & = \int_{0}^{+\infty} \frac{t}{t^{2}+w} e^{-\frac{1}{2}t^{2}} dt + \int_{-\infty}^{0} \frac{t}{t^{2}+w} e^{-\frac{1}{2}t^{2}} dt \\
& = \int_{0}^{+\infty} \frac{t}{t^{2}+w} e^{-\frac{1}{2}t^{2}} dt +\int_{0}^{+\infty}\frac{-t}{t^{2}+w} e^{-\frac{1}{2}t^{2}} dt  = 0.
\end{align*}
\end{proof}

\section{Proof of Equation \eqref{eq:high_freq}} \label{sec:append2}
\begin{proof}
Note that we have $\Sigma_{i=1}^N \hat x_i^2 = \bm{x^Tx}$ and $\Sigma_{i=1}^N \lambda_i\hat x_i^2 = \bm{x^TLx}$ according to spectral graph theory \cite{spielman2007spectral}. The detailed derivation of Equation \eqref{eq:high_freq} is:
\begin{equation*}
    S_{\textnormal{high}}
    =\int_0^{\lambda_N}1-f(t)dt
    =\lambda_N - \sum_{i=1}^N(\lambda_i-\lambda_{i-1})\eta_{i-1}
    =\sum_{i=1}^N\lambda_i(\eta_{i}-\eta_{i-1})
    =\frac{\sum_{i=1}^N\lambda_i \hat x^2_i}{\sum_{i=1}^N \hat x_i^2}
    =\frac{\bm{x^TLx}}{\bm{x^Tx}}
\end{equation*}

\end{proof}



\section{Baselines and Implementation Details} \label{sec:append4}
The first group only considers node features without graph relations:
\begin{itemize}[itemsep=2pt,topsep=0pt,parsep=0pt]
    \item \textbf{MLP}: a multi-layer perceptron network consisting of two linear layers with activation functions.
    \item \textbf{SVM}:~\cite{chang2011libsvm}: a support vector machine with the Radial Basis Function (RBF) kernel.
\end{itemize}
The second group is general GNN models for node classification:
\begin{itemize}[itemsep=2pt,topsep=0pt,parsep=0pt]
    \item \textbf{GCN}~\cite{Kipf}: a graph convolutional network using the first-order approximation of localized spectral filters on graphs.
    \item \textbf{ChebyNet}~\cite{defferrard2016convolutional}: a graph convolutional network which restricts convolution kernel to a Chebyshev polynomial.
    \item \textbf{GAT}~\cite{velivckovic2017graph}: a graph attention network that employs the attention mechanism for neighbor aggregation.
    \item \textbf{GIN}~\cite{xu2018powerful}, a GNN model connecting to Weisfeiler-Lehman (WL) graph isomorphism test.
    \item \textbf{GraphSAGE}~\cite{hamilton2017inductive}: a GNN model based on a fixed sample number of the neighbor nodes.
    \item \textbf{GWNN}~\cite{GWNN}: a graph wavelet neural network using heat kernels to generate wavelet transforms.
\end{itemize}
The third group is state-of-the-art methods for graph-based anomaly detection:
\begin{itemize}[itemsep=2pt,topsep=0pt,parsep=0pt]
    \item \textbf{GraphConsis}~\cite{liu2020alleviating}: a heterogeneous graph neural network which tackles context, feature and relation inconsistency problem in graph anomaly detection.
    \item \textbf{CAREGNN}~\cite{dou2020enhancing}: a camouflage-resistant GNN which enhances the aggregation process with three unique modules against camouflages and reinforcement learning.
    \item \textbf{PC-GNN}~\cite{liu2021pick}: a GNN-based imbalanced learning method to solve the class imbalance problem in graph-based fraud detection via resampling. 
\end{itemize}

MLP is implemented by PyTorch \cite{paszke2019pytorch}, and SVM is in Scikit-learn~\cite{pedregosa2011scikit}. For GCN, ChebyNet, GAT, and GraphSAGE, we use the implementation of DGL\footnote{\url{https://github.com/dmlc/dgl}}. GWNN is implemented by ourselves since the source code does not support the fast algorithm.  For GraphConsis\footnote{\url{https://github.com/safe-graph/DGFraud}}, CAREGNN\footnote{\url{https://github.com/YingtongDou/CARE-GNN}}, and PC-GNN\footnote{\url{https://github.com/PonderLY/PC-GNN}}, we use the code provided by the authors. Our model is implemented based on PyTorch and DGL~\cite{wang2019dgl}. We conduct all the experiments on a high performance computing server running Ubuntu 20.04 with a Intel(R)  Xeon(R) Gold 6226R CPU and 64GB memory. 


As the first and second groups of baselines are not specifically designed for anomaly detection, they suffer from class imbalance and produce very few positive predictions. For fair comparisons, in training, we use the weighted cross-entropy Equation \eqref{eq:method_loss} in BWGNN. In validation, we search for the best binary classification threshold by adjusting it at 0.05 intervals to achieve the best F1-macro.

\section{Additional Experimental Results} \label{sec:append5}

In Table \ref{tab:exp3}, we report the average value and standard deviation of 10 runs on YelpChi and Amazon.

\begin{table*}[h]
\caption{Experimental results (Mean $\pm$ Std.) of compared methods on the YelpChi and Amazon datasets with 1\% and 40\% training ratios.}
\vskip 0.1in
  \resizebox{\linewidth}{!}{
  \begin{tabular}{r|cc|cc|cc|cc}
    \toprule
    Dataset& \multicolumn{2}{c|}{YelpChi (1\%)} & \multicolumn{2}{c|}{YelpChi (40\%)} & \multicolumn{2}{c|}{Amazon (1\%)} & \multicolumn{2}{c}{Amazon (40\%)}\\
    
    Metric & F1-macro & AUC & F1-macro & AUC & F1-macro & AUC & F1-macro & AUC\\
    \midrule
    MLP             & 53.90$\pm$\small{0.23} & 59.83$\pm$\small{0.40} & 57.57$\pm$\small{0.89} & 66.52$\pm$\small{1.09} & 74.68$\pm$\small{1.25} & 83.62$\pm$\small{1.76} & 79.17$\pm$\small{1.26} & 89.80$\pm$\small{1.04}\\
    SVM             & 60.47$\pm$\small{0.24} & 62.92$\pm$\small{0.92} & 70.77$\pm$\small{0.01} & 70.37$\pm$\small{0.04} & 83.49$\pm$\small{1.39} & 81.62$\pm$\small{3.53} & 90.71$\pm$\small{0.04} & 90.51$\pm$\small{0.07}\\ \midrule
    GCN             & 52.48$\pm$\small{0.50} & 54.06$\pm$\small{0.72} & 54.31$\pm$\small{0.77} & 56.51$\pm$\small{1.09} & 67.93$\pm$\small{1.42} & 82.85$\pm$\small{0.71} & 67.47$\pm$\small{0.52} & 83.49$\pm$\small{0.47}\\
    ChebyNet        & 63.13$\pm$\small{0.50} & 73.48$\pm$\small{0.74} & 65.72$\pm$\small{0.48} & 78.19$\pm$\small{0.63} & 85.74$\pm$\small{1.67} & 87.60$\pm$\small{0.61} & 91.94$\pm$\small{0.29} & 94.64$\pm$\small{0.53}\\
    GAT             & 50.27$\pm$\small{2.31} & 50.95$\pm$\small{1.39} & 54.64$\pm$\small{2.19} & 57.20$\pm$\small{0.24} & 60.84$\pm$\small{2.47} & 73.45$\pm$\small{1.26} & 83.18$\pm$\small{2.91} & 89.90$\pm$\small{0.95}\\
    GIN             & 57.57$\pm$\small{1.15} & 64.73$\pm$\small{1.73} & 62.85$\pm$\small{0.76} & 74.09$\pm$\small{1.06} & 68.69$\pm$\small{4.12} & 78.83$\pm$\small{3.82} & 69.26$\pm$\small{2.45} & 80.56$\pm$\small{2.99}\\
    GraphSAGE       & 58.41$\pm$\small{2.12} & 67.58$\pm$\small{1.69} & 65.49$\pm$\small{1.84} & 78.31$\pm$\small{2.14} & 70.78$\pm$\small{3.85} & 75.37$\pm$\small{2.49} & 74.17$\pm$\small{1.37} & 86.95$\pm$\small{2.74}\\
    GWNN            & 59.10$\pm$\small{6.53} & 67.16$\pm$\small{11.44} & 65.29$\pm$\small{6.67} & 75.32$\pm$\small{8.97} & 87.01$\pm$\small{1.98} & 85.37$\pm$\small{2.32} & 91.00$\pm$\small{0.27} & 93.19$\pm$\small{2.22}\\\midrule
    GraphConsis     & 56.79$\pm$\small{2.72} & 66.41$\pm$\small{3.41} & 58.70$\pm$\small{2.00} & 69.83$\pm$\small{3.02} & 68.59$\pm$\small{3.41} & 74.11$\pm$\small{3.53} & 75.12$\pm$\small{3.25} & 87.41$\pm$\small{3.34}\\
    CAREGNN         & 62.18$\pm$\small{1.39} & 75.07$\pm$\small{3.88} & 63.32$\pm$\small{0.94} & 76.19$\pm$\small{2.92} & 68.78$\pm$\small{1.68} & 88.69$\pm$\small{3.58} & 86.39$\pm$\small{1.73} & 90.53$\pm$\small{1.67}\\
    PC-GNN          & 59.82$\pm$\small{1.42} & 75.47$\pm$\small{0.98} & 63.00$\pm$\small{2.30} & 79.87$\pm$\small{0.14} & 79.86$\pm$\small{5.65} & \textbf{90.40$\pm$\small{2.05}} & 89.56$\pm$\small{0.77} & 95.86$\pm$\small{0.14}\\
    \midrule
    BWGNN (Homo)   & 61.15$\pm$\small{0.41} & 72.01$\pm$\small{0.48} & 71.00$\pm$\small{0.91} & 84.03$\pm$\small{0.98} & \textbf{90.92$\pm$\small{0.78}} & 89.45$\pm$\small{0.33} & \textbf{92.29$\pm$\small{0.44}} & \textbf{98.06$\pm$\small{0.45}}\\ 
    BWGNN (Hetero) & \textbf{67.02$\pm$\small{0.50}} & \textbf{76.95$\pm$\small{1.38}} & \textbf{76.96$\pm$\small{0.89}} & \textbf{90.54$\pm$\small{0.49}} & 83.83$\pm$\small{3.79} & 86.59$\pm$\small{2.62} & 91.72$\pm$\small{0.84} & 97.42$\pm$\small{0.48}\\ 
    \bottomrule
\end{tabular}
}
\label{tab:exp3}
\end{table*}


\end{document}